%% file: ref.tex
\documentclass[12pt]{article}
\usepackage{graphicx} 
\usepackage{amsmath,amsthm,amsfonts,amssymb,mathrsfs}
\usepackage{hyperref}
\usepackage{enumitem}
\usepackage{microtype} 
\usepackage{color}
\usepackage{hhline}
\usepackage{authblk}

\usepackage{amsmath}
\usepackage{amssymb}
\usepackage{titletoc}
\usepackage{mathrsfs}
\usepackage{amsthm}
\usepackage{indentfirst}
\usepackage{color}
\usepackage{txfonts}
\usepackage{bbm}
\usepackage{amsmath}
\usepackage{amssymb}
\usepackage{braket}
\usepackage{tabularx}
\allowdisplaybreaks

\input{macros}

\begin{document}

\title{Spherical Analysis of Learning Nonlinear Functionals}



\author{ Zhenyu Yang$^1$,
Shuo Huang$^2$ \thanks{ Corresponding author.
\newline  Email addresses: \texttt{zhy.yang@my.cityu.edu.hk} (Z. Yang), \texttt{shuo.huang@iit.it} (S. Huang),\\ \texttt{hanfeng@cityu.edu.hk} (H. Feng), \texttt{dingxuan.zhou@sydney.edu.au} (D.-X. Zhou).} \hspace*{2pt},
Han Feng$^1$, 
Ding-Xuan Zhou$^3$ 
}

\date{
$^1$ Department of Mathematics, City University of Hong Kong, \\Kowloon, Hong Kong SAR\\
$^2$  Istituto Italiano di Tecnologia, Genoa, Italy\\
$^3$  School of Mathematics and Statistics, University of Sydney,\\ Sydney NSW 2006, Australia
}
	
\maketitle

\begin{center}
\begin{minipage}{13.5cm}
{\small{\textbf{Abstract}}\quad }
In recent years, there has been growing interest in the field of functional neural networks. They have been proposed and studied with the aim of approximating continuous functionals defined on sets of functions on Euclidean domains. In this paper, we consider functionals defined on sets of functions on spheres. The approximation ability of deep ReLU neural networks is investigated by novel spherical analysis using an encoder-decoder framework. An encoder comes up first to accommodate the infinite-dimensional nature of the functional's domain. It utilizes spherical harmonics to help us extract the latent finite-dimensional information of functions, which in turn facilitates in the next step of approximation analysis using fully connected neural networks. Moreover, real-world objects are frequently sampled discretely and are often corrupted by noise. Therefore, encoders with discrete input and those with discrete and random noise input are constructed, respectively. The approximation rates with different encoder structures are provided therein.
\end{minipage}
\end{center}

\vspace{0.2cm}
	
\section{Introduction}
Deep neural networks are a type of machine learning models which have been the most powerful and practical learning algorithms in our day-to-day life. These models are routinely used across multiple domains, such as machine translation \cite{vaswani2017attention}, computer vision \cite{he2016deep}, and human-machine interaction \cite{hinton2012deep}. These applications, although diverse, share a common underlying theme: they all involve data, which are essentially discrete samples derived from continuous functions. From this perspective, using deep neural networks in these contexts can be viewed primarily as a tool for processing samples of continuous objects. Consequently, a critical area of study within this domain is the examination of deep neural networks' ability to approximate function operators. Specifically, this involves estimating the error when approximating an operator that maps one class of functions to another.
	

During the past few years, there has been a large literature demonstrating the approximation ability of deep neural networks for functions on domains $\mathcal{X}$ of Euclidean spaces $\rr^n$.  
Considering the deep neural networks (DNN) with $J$ layers for an input $x \in X \subset \RR^d$, the output of the $J$-th layer is given by
\begin{equation}\label{equ:fnn}
	h^{(J)}(x) = F^{(J)}\cdot \sigma(F^{(J-1)} \cdots  \sigma(F^{(1)}x +b^{(1)}) + \cdots + b^{(J-1)})+b^{(J)},
\end{equation}
where $F^{(j)} \in \RR^{d_j \times d_{j-1}}$ is the connection matrix with width $d_j \in \NN_+$ for the $j$-th layer, $b^{(j)} \in \RR^{d_j}$ is the bias and $\sigma(x)=\max\{0,x\}$ is the rectified linear unit (ReLU) activation function acting componentwise.
Recent studies have demonstrated that deep ReLU neural networks exhibit high expressive powers in approximating functions in Sobolev spaces \cite{yarotsky2017error, schmidt2020nonparametric,suzuki2019adaptivity,siegel2023optimal,yang2024nonparametric}. Especially there are also some results demonstrating the advantages of structured neural networks for specific function classes \cite{poggio2017and,zhou2020universality,fang2020theory,mao2021theory,montavon2011kernel}.

	
Despite the enormous success of deep neural networks in function approximation, the underlying theory for applying neural networks to learning functionals remains largely unexplored, though some attempts have been made.
For the classical functional regression problems \cite{morris2015functional,chen2011single}, splines, wavelets, or link functions are used to learn a regression map from a function to a scalar response and to establish convergence rates. For recent prevalent needs, such as solving partial differential equations (PDEs) using neural networks \cite{lu2021learning,bhattacharya2021model,kovachki2023neural}, where the goal is to learn a map from parametric function space $\mathcal{U}$ to the PDE solution space $\mathcal{V}$, encoder-decoder networks are utilized to extract latent finite-dimensional structures. This is achieved by adopting two encoder/decoder pairs on $\mathcal{U}$ and $\mathcal{V}$ to tackle the infinite-dimensional input and output domains.
One of the first operator network architectures was introduced in \cite{chen1995universal}.
It was extended to a deep learning network, DeepONet, by \cite{lu2021learning} and showed that DeepONet could accurately predict the behavior of complex systems.
And \cite{seidman2022nomad} formulated a novel operator learning architecture with nonlinear decoders.
Recently,  \cite{liu2024generalization} introduced an Auto-Encoder-based Neural Network (AENet), which is efficient for input functions on a nonlinear manifold.
The work \cite{kovachki2024operator} further advanced the field of operator learning by providing a comprehensive analysis of algorithms used in this domain.
From another approximation viewpoint, \cite{song2023approximation(a),song2023approximation(b),zhou2024approximation} gave the quantitative rates of $O\left((\log M)^{\frac{-\beta \lambda}{d}}\right)$ and $O\left((\log M)^{\frac{-\beta s}{d}}\right)$ of approximating Lipschitz-$\lz$ continuous nonlinear functionals and smooth functionals of order $s$ defined on the continuous function $C^\beta([-1,1]^d)$ using deep neural networks in terms of the total number of weights $M$. There have also been attempts to address the curse of dimensionality in functional learning. 
	

The study of neural networks’ capability to approximate nonlinear operators is also critical when dealing with functions defined on non-Euclidean domains, such as spherical surfaces, or when signals are compromised by noise and discretization. These complexities present challenges that necessitate advanced analytical approaches to explore the corresponding approximation results. The choice of a spherical surface as a domain for data is motivated by its practical importance in fields such as planetary sciences, astronomy, and global telecommunications. There have been some results studying the approximation ability of neural networks for estimating functions defined on spheres \cite{fang2020theory,feng2023generalization}.
Moreover, real-world objects are frequently sampled discretely and are often corrupted by noise, which can significantly degrade the quality of data processing and analysis. Addressing these issues requires a robust framework capable of handling the intricacies of discrete and noisy data while maintaining high fidelity.
	
Amid these challenges, we utilize an encoder-decoder framework in this paper, which has been widely adopted in recent advancements in deep learning. Our research then focuses on three pivotal aspects:
\begin{itemize}
    \item The first result concerns orders of approximating continuous operators $F$ defined on smooth function class $W_p^r(\sph)$ on the spherical domain $\sph$ with modulus of continuity satisfying $\omega_F(a)\leq ca^\lambda$. Roughly speaking, we first introduce a linear operator $V_n:L_p(\sph) \to \Pi_{2n}^{d-1}(\sph)$ onto the space $\Pi_{2n}^{d-1}(\sph)$ of polynomials of degree up to $2n$ over the unit sphere $\sph$ as \eqref{equ:Ln_def} followed by an isometric isomorphism map $\phi_n: \Pi_{2n}^{d-1}(\sph) \to \RR^{t_n}$. Then define the encoder ${E} = \phi_n \circ V_n : L_p(\sph) \to \RR^{t_n}$ given by \eqref{equ:ecd}, which is utilized to extract some finite-dimensional information of the impact function $f$. After that, the encoder-decoder network framework is constructed by $D \circ {E}$ with $D:\RR^{t_n} \to \RR$  being a deep neural network. For a more detailed statement, please refer to \thmref{thm:thm1}. {{This result can also be extended to more general settings as illustrated by Remark \ref{rem:rem1}.}}
    \item Second, from a practical application perspective, the input functions are often discretized, which motivates us to consider the discrete input for the encoder. \thmref{thm:discreteinput} provides approximation estimates associated with discrete inputs. The encoder $\widehat{E}:\RR^{M} \to \RR^{t_n}$ is given by $\widehat{E} = \phi_n \circ \widehat{V}_{n,M}$, where $\widehat{V}_{n,M}$ given by \eqref{def:dis-l} is defined on the function values $\{f(\xi_1), \ldots, f(\xi_M)\}$ at some spherical points $\{\xi_j\}_{j=1}^M$ admitting the cubature rule rather than on the function itself. Then $D \circ \widehat{E}$ composes our neural network for approximating the continuous functional.
    {Remark \ref{lem:lem2} demonstrates that this result can be extended to other manifolds given certain conditions.}
    \item In the third aspect, we analyze the approximation accuracy under the impact of noise on inputs. Considering a set of random variables  $\epsilon=\{\epsilon_j\}_{j=1}^{M'}$ with zero mean and values ranging in $[-1,1]$, the input for the encoder $\widetilde{E}$ now becomes $\{\epsilon_j+f(\xi_j)\}_{j=1}^{M'}$. Consequently, the domain of the decoder $D$ is now also composed of random variables, and the approximation rate given by \thmref{thm:discretenoise} is illustrated in a probabilistic sense.
\end{itemize}
By systematically analyzing the approximation stages and the associated error metrics, our approach aims to provide a comprehensive toolkit for researchers and practitioners to understand the power of deep neural networks in dealing with functionals on spherical data. 
	
This paper is structured as follows: Section \ref{main} describes our main results for different encoder-decoder approximation frameworks for continuous operators $F$ defined on $W_p^r(\sph)$ in detail, { {demonstrating sufficient conditions for accuracy estimation}}. Section \ref{proof} presents the proofs for different theorems, propositions, and lemmas. Section \ref{sec:conclusion} is a summary of the paper. Table \ref{table:notation}, located at the end, includes some key notations used throughout.

\section{Main results}\label{main}
In this article, we mainly demonstrate the approximation ability of deep neural networks with an encoder-decoder framework for continuous functionals $F$ defined on some smooth function classes over the (closed) unit sphere $\sph \subset \RR^d$.
The estimation rate of functional $F$ defined on the Sobolev space $W_p^r(\sph), 1 \leq p \leq \infty, r>0,$ over a spherical domain will be illustrated first by our encoder-decoder network in Section \ref{subsec:SobSphere}. {{Actually, the argument is also applicable in a broader functional space, as long as some very general conditions are met. Examples of some common function spaces are also provided in Remark \ref{rem:rem1}.}} Moreover, as functions in real-world applications are frequently sampled discretely and are often corrupted by noise, we construct the encoder $\widehat{E}: \RR^M \to \RR^{t_n}, M,t_n \in \ZZ_+,$ with discrete input, i.e., function values at some spherical points and $\widetilde{E}: \RR^{M'} \to \RR^{t_n}, M' \in \ZZ_+$ underlying discrete input with noise, and give the approximation rates in subsection \ref{subsec:discrete} and subsection \ref{subsec:dis&noise}, respectively.

\subsection{Estimation for continuous functional defined on Sobolev space}\label{subsec:SobSphere}
In this part, we demonstrate the estimation rates for the continuous functional $F$ defined on a Sobolev space $W_p^r(\sph)$ underlying the unit sphere $\sph$ with respect to the total number of trainable parameters $\CN$ of the neural network. 
Tools from spherical harmonics and approximation theory play an essential role. In particular, the following preliminaries are needed.

	
For $1\leq p\leq \infty$, we denote by $L_p(\mathbb{S}^{d-1})=L_p(\mathbb{S}^{d-1},\s_d)$ the $L_p$-function space defined with respect to the normalized Lebesgue measure $\s_d$ on $\mathbb{S}^{d-1}$, and $\|\cdot\|_p$ the norm of $L_p(\sph)$.

A spherical harmonic of degree $k\in\ZZ_+$ on $\sph$ is the restriction to $\sph$ of a homogeneous and harmonic polynomial of total degree $k$ defined on $\RR^{d}$.
Let $\mathcal{H}_k^d$ denote the set of all spherical harmonics of degree $k$ on $\sph$. It is well-known in \cite{dai2013approximation} the dimension of the linear space $\mathcal{H}_k^d$ is
\begin{align*}
    C(k,d) =   {\binom{k+d-1}{k}}-{\binom{k+d-3}{k-2}}=O(k^{d-2}).
\end{align*}

Observe that $L_2(\mathbb{S}^{d-1})$ forms a Hilbert space. For any functions $f$ and $g$ in $L_2(\mathbb{S}^{d-1})$, their inner product is defined as
$
\langle f, g \rangle_2 := \int_{\mathbb{S}^{d-1}} f(x) g(x) \, \md\sigma_d(x).
$
The subspaces $\mathcal{H}_k^d$, where $k \in \mathbb{Z}_+$, correspond to spherical harmonics and are mutually orthogonal with respect to this inner product in $L_2(\mathbb{S}^{d-1})$. Since spherical polynomials are dense in $L_2(\mathbb{S}^{d-1})$, every function $f \in L_2(\mathbb{S}^{d-1})$ can be expanded as a series of spherical harmonics,
\[
f = \sum_{k=0}^\infty \text{Proj}_k f = \sum_{k=0}^{\infty} \sum_{l=1}^{C(k,d)} \widehat{f}_{k,l} Y_{k,l},
\]
where the convergence is in the $L_2(\mathbb{S}^{d-1})$ norm. Here, $\{Y_{k,l}\}_{l=1}^{C(k,d)}$ forms an orthonormal basis for $\mathcal{H}_k^d$, and $\widehat{f}_{k,l}$ are the Fourier coefficients of $f$ defined by
\begin{align}\label{equ:integral_hatf}
    \widehat{f}_{k,l} := \langle f, Y_{k,l} \rangle_{L_2(\mathbb{S}^{d-1})} := \int_{\mathbb{S}^{d-1}} f(x) Y_{k,l}(x) \, \md\sigma_d(x).
\end{align}

The spaces $\mathcal{H}_k^d$  of spherical harmonics can also be characterized as
eigenfunction spaces of the Laplace-Beltrami operator $\Delta_0$ on
$\sph$. Indeed,
\begin{align*}
    \mathcal{H}_k^d=\{f \in C^2(\mathbb{S}^{d-1}):\Delta_0 f=-\l_k f\},
\end{align*}
where $\lambda_k=k(k+d-2)$ and $C^2(\mathbb{S}^{d-1})$ denotes the space of all twice continuously differentiable functions on $\mathbb{S}^{d-1}$.
Now we define the Sobolev space $W_p^r(\mathbb{S}^{d-1})$ to be a subspace of $L_p(\mathbb{S}^{d-1})$, $1\leq p\leq \infty$, $r>0$, with the finite norm
\begin{align}\label{Sobolev_norm}
    \| f \|_{W_p^r(\mathbb{S}^{d-1})}:=\left\| (-\Delta_0+I)^{r/2} f\right\|_p 
    {\asymp} \sum_{k=0}^{\infty}(1+\lambda_k)^{\frac{r}{2}} \left(\sum_{l=1}^{C(k,d)}|\widehat{f}_{k,l}|^p\right)^{\frac{1}{p}}.
\end{align}
Here $A \asymp B$ means there are positive constants $C_1, C_2 $  such that $C_1 B\leq A\leq C_2 B$. 

The main step of spherical approximation is to introduce an essential linear operator $$V_n: L_p(\sph) \to \Pi_{2n}^{d-1}(\sph)$$
mapping from $L_p$ to polynomials of degree up to $2n$.
\begin{definition}\label{def:Vn}
Given a $C^\infty\left([0,\infty)\right)$ function $\eta$ with $\eta(t) = 1$ for $0 \leq t \leq 1$ and $\eta(t) = 0$ for $t \geq 2$, we define a sequence of linear operators $V_n$, $n\in\ZZ_+$, on $L_p(\mathbb{S}^{d-1})$  with $1\leq p \leq \infty$ by
\begin{align} \label{equ:Ln_def}
V_n(f)(x) :=&\sum_{k=0}^\infty \eta\left(\f k n\right)\proj_k(f)(x)=
\sum_{k=0}^\infty\eta\left(\f k n\right) \sum_{l=1}^{C(k,d)}\widehat{f}_{k,l}Y_{k,l}(x) \notag \\
= & \sum_{k=0}^{2n} \sum_{l=1}^{C(k,d)}\eta\left(\f k n\right)\widehat{f}_{k,l}Y_{k,l}(x).
\end{align}
\end{definition}
For simplicity, we replace the multi-index set $\{k,l\}$ by the usual $\{1,2, \ldots, t_n\}$ with an order arranged by the total degree and $t_n=\sum_{k=0}^{2n} \sum_{l=1}^{C(k,d)} 1 = O(n^{d-1})$. The following lemma shows that the linear operator $V_n$ is near-best optimal for $f \in W_p^r(\sph)$.

\begin{lem}\label{lem:bestapproximationerror}
For $n\in \NN$, $r>0$, $1\leq p\leq \infty$ and $f\in W^r_p\left(\mathbb{S}^{d-1}\right)$, there holds
\begin{equation*}\label{equation:bestapproximationerror}
	\left\|f - V_n(f)\right\|_{L_p(\sph)} \leq C_3 n^{-r} \left\|f\right\|_{W^r_p(\sph)},
\end{equation*}
{where $C_3$ is a constant depending on $d$.}
\end{lem}
	
Define an isometric isomorphism $\phi_n: (\Pi_{2n}^{d-1}(\sph),\|\cdot\|_2) \to (\RR^{t_n}, |\cdot|_2)$ by 
\begin{equation}\label{equ:phi_n}
\phi_n({V}_{n}(f)) = \left(\langle {V}_{n}f, Y_1\rangle, \langle {V}_{n}f, Y_2 \rangle, \ldots, \langle {V}_{n}f, Y_{t_n} \rangle\right)',   
\end{equation}
which plays a significant role in converting the problem into a finite-dimensional representation. Leveraging this crucial property, the encoder ${E} : L_p(\sph) \to \RR^{t_n}$ can be subsequently defined  by
\begin{equation}\label{equ:ecd}
{E}(f) = {E}_n(f): = \phi_n\circ V_n(f). 
\end{equation}
Moreover, define $\mu_{F,\phi_n}:=F \circ \phi_n^{-1}: \RR^{t_n} \to \RR$,  then
\begin{equation} \label{equ:muE}
 F(V_n(f)) =F \circ \phi_n^{-1} \circ \phi_n(V_n(f))=\mu_{F,\phi_n}(\phi_n(V_nf))=\mu_{F,\phi_n}({E}(f)).   
\end{equation}
{
Note that the Sobolev norm of $f\in W_p^r(\sph)$ given by \eqref{Sobolev_norm} can be bounded by
\begin{equation*}
\begin{split}
	\sum_{k=0}^{\infty}(1+\lambda_k)^{\frac{r}{2}} \left(\sum_{l=1}^{C(k,d)}|\widehat{f}_{k,l}|^p\right)^{\frac{1}{p}} 
		\asymp  \sum_{k=0}^\infty c_d k^r \left( \sum_{\ell=1}^{C(k,d)}|\widehat{f}_{k,l}|^p\right)^{\frac{1}{p}},
\end{split}    
\end{equation*}
which together with the properties of $\eta(x)\in [0,1]$ and the basis $Y_{k,l}$ induces for $ k = 1,\ldots, \infty, \, l =1,2,\ldots, C(k,d)$, 
\begin{equation}\label{equ:nnDomain}
\left| \langle  V_nf,Y_{k,l}\rangle\right|\leq\left|\widehat{f}_{k,l}\right| \leq { \|f\|_{W_p^r(\sph)}}. 
\end{equation}}
Therefore, if we consider $f \in K := \{f \in W_p^r(\sph): \|f\|_{W_p^r(\sph)}\leq 1\}$ being a unit ball of $W_p^r(\sph)$, $\mu_{F,\phi_n}$ given by \eqref{equ:muE} is a continuous function with the input being a vector belonging to $[-1, 1]^{t_n}$ instead of a function. This highlights the role of our encoder structure, leading us to the second stage of our analysis.

Numerous results provide the approximation rate for continuous functions using deep neural networks \eqref{equ:fnn} in terms of the modulus of continuity \cite{yarotsky2018optimal}. The \textbf{modulus of continuity} for a functional $F: K \to \RR$ on a subset $K$ of $L_p(\sph)$ is given by
\[\omega_F(r)  = \sup\{|F(f_1)-F(f_2)|: f_1, f_2 \in K, \|f_1-f_2\|_p \leq r \}, \quad r>0 .\] 
It is well known that the modulus of continuity $\omega_F$ is a non-decreasing function.
The following result demonstrates the approximation rate with our neural networks $D\circ {E}$ for continuous functional $F$ defined on $W_p^r(\sph)$ with respect to the total number of free parameters $\CN$ of deep neural networks.

\begin{thm}\label{thm:thm1}
Let $d \in \NN_+, r>0, 1\leq p \leq \infty, \lambda>0$. If $F:W_{p}^r(\sph) \to \RR$ is a continuous functional with modulus of continuity $\omega_F(a)\leq ca^\lambda$, $K$ is the unit ball of $W_{p}^r(\sph)$, then there exists an encoder ${E}$  given by \eqref{equ:ecd}  and  a function $D$ generated by a deep ReLU neural network \eqref{equ:fnn}  with total number of parameters $\CN$  such that
\begin{equation*}
  \sup_{f\in K}|F(f)-D \circ {E}(f)| = O\left(\left(\frac{\log(\log \CN)}{\log \CN} \right)^{\frac{\lambda r}{d-1}}\right).
\end{equation*}
\end{thm}
The proof of Theorem \ref{thm:thm1} is followed directly from Proposition \ref{prop:thm1} by  taking
$$n  \asymp\left(\left(\frac{\log \CN}{\log (\log \CN)}\right)^{\frac{1}{d-1}}\right).$$

\begin{prop}\label{prop:thm1}
Let $d,n \in \NN_+, r>0, 1\leq p \leq \infty$ and ${t_n}=(2n+1)^{d-1}$. If $F:W_{p}^r(\sph) \to \RR$ is a continuous functional with modulus of continuity $\omega_F$, $K$ is the unit ball of $W_{p}^r(\sph)$, then there exists an encoder ${E}$ given by \eqref{equ:ecd}  and  a function $ D$ generated by a deep ReLU neural network \eqref{equ:fnn} with total number of parameters $\CN$ such that
\[\sup_{f\in K}|F(f)- D\circ {E}(f)| 
\leq \omega_F\left(C_3 n^{-r}
\right)+{ C_{n,d}}\omega_F\left(C_{n,d}n^{(d-1)(\frac{1}{2}-\frac{1}{p})}\CN^{-\frac{2}{t_n}}\right),\]
where $C_3$ is a constant depending only on $d$, $C_{n,d}$s are constants depending on $n$ and $d$ and may vary in different positions.
\end{prop}
The proof of Proposition \ref{prop:thm1} is divided into two steps and can be found in Section \ref{subsec:proof_sob}. The first step is to estimate the error of $F(f) - F(V_n f)$ utilizing Lemma \ref{lem:bestapproximationerror}, followed by analyzing the difference $F(V_n f) - D \circ {E}(f)$ adopting the estimation of approximating continuous functions using deep ReLU neural networks \cite{yarotsky2018optimal}.


\begin{remark}\label{rem:rem1}
The result can be generalized to functional spaces with broader domains as long as they satisfy the following conditions. 
Suppose $\Omega \subset \RR^d$, and $\{ \varphi_j,\psi_j\}_{j\in{\nn_+}}$ be a family of functions in $L_2(\Omega)$.
Let $K$ be a subset of $L_2(\Omega)$ with $\max_{j}|\langle f,\psi \rangle|<\fz$ and satisfy the following two conditions:
\begin{enumerate}[label=(\roman*)]
\item \label{asmp:1} 
There exists some $\alpha>0$, such that the operator $\{V_t\}$ defined by $V_t(f) = \sum_{j=1}^t \langle f,\varphi_j \rangle \psi_j$ satisfy
\begin{equation*}\label{equ:AppAsmpt}
\|V_t(f) - f \|_2 \leq c_1 t^{-\alpha},\ \forall f\in K,\ t\in\nn.
\end{equation*}
\item \label{asmp:2}
$\{\psi_j\}$ is a Bessel sequence, which implies
there exists some constant $c_2>0$ such that
\begin{equation*}\label{equ:PsvAsmpt}
\left\|\sum_{j=1}^t a_j \psi_j\right\|_2 \leq c_2 \|a\|_2, \forall a\in\rr^t,\ t\in\nn.
\end{equation*}
\end{enumerate}
Then for a continuous functional $F: K \to \RR$ with modulus of continuity $\omega_F$, there exists an encoder-decoder framework defined similarly by \eqref{equ:ecd} and \eqref{equ:fnn} with number of free parameters $\CN$, such that
\begin{equation}
    \sup_{f \in K}|F(f)-D \circ E(f)| \leq \omega_F\left(\max\left\{c_1t^{-\alpha}, c_2\CN^{-\frac{2}{t}}\right\}\right).
\end{equation}
\begin{example}[
Orthogonal Basis]
Take $\varphi_j=\psi_j$ for all $j\in\nn$, with $\varphi_j=\psi_j=Y_{k,l}$ being the orthogonal basis of spherical harmonics. \lemref{lem:bestapproximationerror} verifies the validity of the assumption with $t = O(n^{d-1})$ and $\alpha = {r}/{d}$. 
\end{example}

\begin{example}[Biorthogonal Basis]
In biorthogonal wavelets setting, $\{\varphi_j\}$ are decomposition (analysis) wavelets and $\{\psi_j\}$ are the reconstruction (synthesis) sets. Results in \cite{dung2011optimal,suzuki2019adaptivity} also show that $\alpha = {r}/{d}$ for approximating Besov spaces with smoothness $r>0$.
\end{example}



\end{remark}

\subsection{Analysis for the encoder with discrete input}\label{subsec:discrete}
As objects in real-world applications are frequently sampled discretely for processing, instead of considering the encoder $E$ with input $f \in W_p^r(\sph)$, the key concept in this section is to use discrete function values at certain spherical points as input. 
The crucial idea is to transform the linear operator $V_n$.
Instead of relying on the inner product of $f$, 
we introduce a novel linear operator $\widehat{V}_{n,M}(f)$ relies solely on some function values of the function $f$ on the unit sphere.
This operator is primarily inspired by \cite{wang2017filtered}. 
In the following section,
the coefficients $\widehat{f}_{k,l}$ are the integral of the function $f$ with the orthonormal basis $\{Y_{k,l}\}$ as introduced in Definition \ref{def:Vn}.
To overcome the continuous occurrences of $f$, we shall use a cubature formula for integration of polynomials of degree $2m$ on $\sph$, $d\geq 3$, see \cite[Theorem 3.1]{brown2005approximation}, to convert it into summation.
	
\begin{lem}\label{cubature}
For any integer $m > 0$, there exists $M \asymp m^{d-1}$ distinct points $\{\xi_j\}_{j=1}^M$ on $\mathbb{S}^{d-1}$ and positive numbers $\lambda_j \asymp m^{-(d-1)}$ for $j = 1, \dots, M$, such that for every $f \in \Pi_{2m}^{d-1}$, we have
\begin{equation}\label{equ:cub}
\int_{\mathbb{S}^{d-1}} f(x)\,\mathrm{d}\sigma_d(x)
= \sum_{j=1}^{M} \lambda_j f(\xi_j).
\end{equation}
Moreover, there holds the Marcinkiewicz-Zygmund inequality,
\begin{equation}\label{equ:m-z}
\|f\|_{p} \asymp\left\{\begin{array}{ll}
\left(\sum\limits_{j=1}^M \lambda_{j}|f(\xi_j)|^{p}\right)^{\frac{1}{p}}, & \text { if } 1\leq p<\infty, \\
\\
\max\limits_{j=1,\ldots, M}m^{d-1} \lambda_{j}|f(\xi_j)|, & \text { if } p=\infty.
\end{array}\right.
\end{equation}
Particularly, we call such a family $\{(\l_j, \xi_j)\}_{j=1}^M$ that follows a cubature rule of degree {$2m$}.

\end{lem}

By discretizing the inner product \eqref{cubature} using the above cubature rule with $\{(\lambda_j,\xi_j)\}_{j=1}^M$ following the cubature rule of degree $2m$, 
we define the \textbf{discrete linear operator $\widehat{V}_{n,M}$} with $M \asymp m^{d-1}$ as
\begin{align}\label{def:dis-l}
\widehat{V}_{n,M}(\bz_f)\notag
&:=\sum_{k=0}^\fz \sum_{l=1}^{C(k,d)}
{\eta}\lf(\frac{k}{n}\r) \lf( \sum_{j=1}^M \lambda_j f(\xi_j)Y_{k,l}(\xi_j) \r)Y_{k,l}\\
&=\sum_{k=0}^\fz \sum_{l=1}^{C(k,d)}
{\eta}\lf(\frac{k}{n}\r){\braket{f, Y_{k, l}}}_{Q_m} Y_{k, l},
\end{align}
where $\bz_f = (f(\xi_1), f(\xi_2), \ldots, f(\xi_M))' \in \RR^M$ and  $\braket{f,g}_{Q_m}:= Q_m(fg)$ is the discrete version of the inner product for $f,\,g\in L^p$.


Accordingly, we define the  \textbf{discrete encoder} $\widehat{E} : \RR^M \to \RR^{t_n}$ by
\begin{equation}\label{eq:ecd_dis}
	\widehat{E}(\beta_f) = \widehat{E}_n(\beta_f):= \phi_n \circ \widehat{V}_{n,M}(\beta_f),
\end{equation}
which serves as the input to the subsequent classical neural network $D$, and $\phi_n$ is an isometric isomorphism map given by \eqref{equ:phi_n}.
Since ${\eta}(x)\in[0,1]$, $\{(\lambda_j,\xi_j)\}_{j=1}^M$ obeys the cubature rule \eqref{equ:cub} and Marcinkiewicz-Zygmund inequality \eqref{equ:m-z}, it follows directly from \eqref{equ:nnDomain} that the absolute value of each component in $\widehat{E}(\beta_f)$ can be bounded by
\begin{equation}\label{equ:c_bound}
\begin{split}
\left|{\eta}\lf(\frac{k}{n}\r){\braket{f, Y_{k, l}}}_{Q_m}\right|
\leq \left|{\braket{f, Y_{k, l}}}_{Q_m}\right|
= \left|\int_{\sph} f(x)Y_{k,l}(x)\,\md \s_d(x)\right|
= \left|\widehat{f}_{k,l}\right| \leq \|f\|_{W_p^r(\sph)}.
 \end{split}
\end{equation}

The following theorem presents a rate of approximating a functional $F$ by the neural networks and the encoder with discrete input.

\begin{thm}\label{thm:discreteinput}
Let $d,n\in \NN_+, 1\leq p \leq \infty$, $r>{(d-1})/{p}$,  ${t_n}=(2n+1)^{d-1}$.
If $F:W_{p}^r(\sph) \to \RR$ is a continuous functional with modulus of continuity $\omega_F$, $K$ is the unit ball of $W_{p}^r(\sph)$,
then there exists an encoder $\widehat{E}$ given by \eqref{eq:ecd_dis} with $m\geq{(3n-1)}/{2}$ and a function $D$ generated by a deep ReLU neural network \eqref{equ:fnn} with total number of parameters $\CN$ such that
\[
\sup_{f\in K}|F(f)- D \circ \widehat{E}(\bz_f)| 
\leq \omega_F\left({C_4 n^{-r}}
\right)
+{C_{n,d}}\omega_F\left(C_{n,d}n^{(d-1)(\frac{1}{2}-\frac{1}{p})}\CN^{-\frac{2}{t_n}}\right),
\]
where $C_4$ is a constant depending only on $d$, $C_{n,d}$s are constants depending on $n$ and $d$ and may vary in different positions.
\end{thm}  

{
The following remark shows the above result in functional learning could be applicable to general manifolds as long as they satisfy specific assumptions. The conditions are given by \cite{montufar2022distributed}, which was inspired by \cite{maggioni2008diffusion}.
\begin{remark}\label{lem:lem2}
Let $d\geq1$, $\CM$ be a compact and smooth Riemannian manifold of dimension $d$ with smooth or empty boundary,
$\mu$ is the Riemannian measure normalized with the total volume $\mu(\CM)=1$ and $d(x,y)$ is the geodesic distance induced by the Riemannian metric.
Suppose $(\CM,\mu)$ satisfy the following conditions:
\begin{enumerate}[label=(\roman*)]
\item For $c \geq c' >0,$ depending only on $d$ and $mu$, for all $\beta > \alpha>0$ and $x\in\CM$,
    \[\mu(\mathcal{B}({x}, \alpha)) \leq c \alpha^d, \quad \mu(\mathcal{B}({x}, \alpha, \beta))= \mu(\mathcal{B}({x}, \beta)) -\mu(\mathcal{B}({x}, \alpha)) \leq c'(\beta^d - \alpha^d),
\]
where $\mathcal{B}({x}, \alpha)= \{{y} \in \mathcal{M}\,|\,d({x}, {y}) \leq \alpha\}$ is a ball with center $x$ and radius $\alpha$.
\item For \(\ell, \ell' \in \mathbb{N}_+\), the product of eigenfunctions \(\phi_\ell \in \Pi_n,\,\phi_{\ell'} \in \Pi_{n'}\) for the Laplace–Beltrami operator \(\Delta\) on \(\mathcal{M}\) is a polynomial of degree \(n + n'\).
\end{enumerate}
Those two assumptions could be satisfied by many classical manifolds, like hypercubes \([0, 1]^d\), unit spheres and balls in real or complex Euclidean coordinate spaces, and also graphs equipped with an atomic measure on the graph nodes with the graph Laplacian defined as the difference of the identity matrix and the adjacency matrix. For more details, see in \cite{montufar2022distributed}.
\end{remark}
}

The proof of \thmref{thm:discreteinput} is given in Section \ref{subsec:proof_dis}, which follows a similar idea to that of Proposition \ref{prop:thm1}.
The main difference lies in the description of the linear operator $\widehat{V}_{n,M}$ and its approximation property given by Lemma \ref{lem:discrete_ln_app} below.
		
\subsection{Estimation for the encoder with discrete and noisy input}\label{subsec:dis&noise}
In practical applications, the data collected often demonstrates inherent randomness and inevitably contains noise, so it is important to consider data with noise as input of the encoder to address real-world conditions and ensure robust performance better. In this subsection, 
we consider a scenario when the input of the encoder has some random noise.

Let {$\epsilon=\{\epsilon_j\}_{j=1}^{M'}$} represent the noise, where $\{\epsilon_j\}$ are independent random variables with means of $\EE[\epsilon_j]=0$ and values ranging within $[-1,1]$ for $j=1,\ldots,M'$. 
For $\lf\{\lf(\lambda_j,\xi_j \r)\r\}_{j=1}^{M'}$ admitting the cubature rule of degree $2m$ and $M' \asymp m^{d-1}$, the discrete operator $\widetilde{V}_{n,M'}$ defined on $\beta_{f,\epsilon}:=\{\epsilon_j+f(\xi_j)\}_{j=1}^{M'}$  is given by
\begin{align}\label{def:dis-n}
\widetilde{V}_{n,M'}\left(\beta_{f,\epsilon}\right)
:= \sum_{k=0}^\infty \sum_{l=1}^{C(k,d)} {\eta}\lf(\frac{k}{n}\r)\lf( \sum_{j=1}^{M'}\lambda_j\lf(f(\xi_j)+\epsilon_j\r)
Y_{k,l}(\xi_j)\r)Y_{k,l}.
\end{align}
Accordingly, we define the \textbf{encoder for discrete and noisy input} by
\begin{equation}\label{equ:ecd_dis&noise}
\widetilde{E}\left(\beta_{f,\epsilon}\right) = \widetilde{E}_n\left(\beta_{f,\epsilon}\right) := \phi_n \circ \widetilde{V}_{n,M'}\left(\beta_{f,\epsilon}\right).
\end{equation}

Note that $\int_{\sph}Y_{k,l}(x) \,\md \s_d(x) = \text{Vol}(\sph)$ for $k=0$ otherwise $0$, where $\text{Vol}(\sph)$ represents the surface area of the unit sphere. 
By the definition \eqref{equ:phi_n} of $\phi_n$, each component of $\widetilde{E}\lf( \beta_{f,\epsilon}\r)$ given as ${\eta}\lf(\frac{k}{n}\r)\lf( \sum_{j=1}^{M'}\lambda_j\lf(f(\xi_j)+\epsilon_j\r)
Y_{k,l}(\xi_j)\r)$ can be bounded by
\begin{align*}
&\lf|\lf( \sum_{j=1}^{M'}\lambda_j\lf(f(\xi_j)+\epsilon_j\r) Y_{k,l}(\xi_j)\r) \r|\\
&\quad\leq \lf|\int_{\sph} f(x)Y_{k,l}(x)\,\md \s_d(x)\r|+ \lf| \int_{\sph}Y_{k,l}(x)\,\md \s_d(x)\r|\\
&\quad\leq\|f\|_{W_p^r(\sph)} + \text{Vol}(\sph):=R',
\end{align*}
where we use Marcinkiewicz-Zygmund inequality \eqref{equ:m-z} in the first inequality and \eqref{equ:nnDomain} in the last step. Therefore, the encoder $\widetilde{E}$ induces a $t_n$-dimensional vector with values ranging in $[-R',R']$.

The following lemma shows that $\widetilde{V}_{n,M'}(\bz_{f,\epsilon})$ presents a good rate of approximating $f$ in high probability, and the proof can be found in Section \ref{proof_subsec:dis&noise}.

\begin{lem}\label{lem:dis&noise}
Let $d,n \in \NN_+ ,\dz >0$  and $2m \in \NN_+, m \geq n$ satisfy
\begin{equation}\label{equ:m'}
m^{d-1} \geq (\delta+d-1) n^{2r+d-1}\log n.
\end{equation}
Suppose $M'\asymp m^{d-1}$ and $\{(\lz_j, \xi_j)\}_{j=1}^{M'}$ admits a cubature rule of degree {$2m$}.
Then, for all $f\in W^r_p(\sph)$, $1\leq p \leq \infty, r>{(d-1)}/{p}$,
we have 
\begin{align}\label{3.1bb}
\Prob \lf(\, \lf\|\, \widetilde{V}_{n,M'}( \bz_{f,\epsilon})-f \,\r\|_p 
\geq { C_5} n^{-r}\r) 
\leq { C_6} n^{-\delta}.
\end{align}
{ Here, ${C_5}$ and $C_6$ are positive constants independent of the distribution of $\epsilon$.}
\end{lem}

\begin{thm}\label{thm:discretenoise}
Let $d,n\in \NN_+, 1\leq p \leq \infty$, $r>{(d-1})/{p}$,  ${t_n}=(2n+1)^{d-1}$ and $\delta>0$. Let {$\epsilon=\{\epsilon_j\}_{j=1}^{M'}$} represent the noise, where $\{\epsilon_j\}$ are independent random variables with means of $\EE[\epsilon_j]=0$ and values ranging within $[-1,1]$. 
If $F:W_{ p}^r(\sph) \to \RR$ is a continuous functional with modulus of continuity $\omega_F$, $K$ is the unit ball of $W_{p}^r(\sph)$, 
then there exists an encoder $\widetilde{E}$ given by \eqref{equ:ecd_dis&noise} with $m$ satisfying \eqref{equ:m'} and a function $D$ generated by a deep ReLU neural network \eqref{equ:fnn} with total number of parameters $\CN$, such that the following holds with confidence $1-C_6n^{-\delta}$,
\begin{equation*}
\sup_{f \in K}\lf| F(f)- D \circ \widetilde{E}(\bz_{f,\epsilon})\r |
\leq C_{n,d}\omega_F\left(C_{n,d,R'}n^{(d-1)(\frac{1}{2}-\frac{1}{p})}\CN^{-\frac{2}{t_n}}\right)
+\omega_F(C_5n^{-r}).
\end{equation*}
Here, ${C_5}$ and $C_6$ are positive constants independent of the distribution of $\epsilon$, $C_{n,d}$ is a constant depending on $n$ and $d$, and $C_{n,d,R'}$ is depending on $n,d$ and $R'$.
\end{thm}
	
\section{Proof of main results}\label{proof}
In this section, we demonstrate the proof of our main results and our analysis framework.

\subsection{Proof of functionals defined on the Sobolev space}\label{subsec:proof_sob}
In this subsection, we prove Proposition \ref{prop:thm1}. As shown in Subsection \ref{subsec:proof_sob}, we get a finite-dimensional vector after the isometric isomorphism mapping $\phi_n$,  whose scope determines the range of the input of the fully connected neural network $D$.
	
As ${V}_{n}(f)=\sum_{k=0}^{2n} \sum_{l=1}^{C(k,d)}\eta\left(\f k n\right)\widehat{f}_{k,l}Y_{k,l}$ and reordering $\{k,l\}$ by the usual $\{1,2, \ldots, t_n\}$,  the range of each component of
\begin{equation*}
		\phi_n({V}_{n}(f)) = \left(\langle {V}_{n}f, Y_1\rangle, \langle {V}_{n}f, Y_2 \rangle, \ldots, \langle {V}_{n}f, Y_{t_n} \rangle\right)'
\end{equation*}
{will be bounded by $\|f\|_{W_p^r(\sph)} \leq 1$ as given by \eqref{equ:nnDomain}.}
Recall that 
$$F(V_n(f)) = F\circ \phi_n^{-1}\circ \phi_n(V_n(f)) =  \mu_{F,\phi_n}(\phi_n(V_n(f)))=\mu_{F,\phi_n}({E(f)}).$$

Next, following a result in \cite{yarotsky2018optimal}, we have that for a continuous function $f: [-1,1]^{t_n} \to \RR$ with the modulus of continuity $\omega_f$, there exists an output function $D$ generated by an equal-width fully-connected ReLU neural network and total number of parameters $\CN$, such that, 
\begin{equation}\label{equ:app_ctns}
   \| f - D \|_{\infty} \leq { C_{n,d}}\omega_f(C_{n,d}\CN^{-\frac{2}{t_n}}).
\end{equation}
We can adopt this to estimate the approximation error between the continuous function $\mu_{F,\phi_n}$ and a deep neural network $D$. 
 
The following proposition presents the relationship between the modulus of continuity of $\mu_{F,\phi_n}$ and that of $F$.
	
\begin{prop}\label{prop:w-w}
	Let $\omega_{\mu_{F,\phi_n}}$ be the modulus of continuity of $\mu_{F,\phi_n}$. Then for $r>0$ and $n\in\nn$, we have
	\[\omega_{\mu_{F,\phi_n}}(r)\leq \omega_F\left(C_d n^{(d-1)(\frac{1}{2}-\frac{1}{p})}r\right), \]
 where $C_d$ is a constant depending only on d.
\end{prop}

\begin{proof}
For $y_1,y_2 \in \RR^{t_n}$, by the definition of $\mu_{F,\phi_n}$, there holds
  \begin{align*}
		|\mu_{F,\phi_n}(y_1)-\mu_{F,\phi_n}(y_2)| &=|F(\phi_n^{-1}y_1)-F(\phi_n^{-1}y_2)|\\
		&\leq \omega_F\left(\|\phi_n^{-1}(y_1-y_2)\|_p\right)\\
		&\leq\omega_F\left(C_dn^{(d-1)(\frac{1}{2}-\frac{1}{p})}\|\phi_n^{-1}(y_1-y_2)\|_2\right)\\
      &\leq\omega_F\left(C_dn^{(d-1)(\frac{1}{2}-\frac{1}{p})}|y_1-y_2|_2\right),
  \end{align*}
where we use the estimate $\|P\|_p\leq C_d n^{(d-1)(\frac{1}{2}-\frac{1}{p})}\|P\|_2$ of the $L_p$-norm for $P\in\Pi_{2n}$ in terms of the $L_2$-norm, with a constant $C_d$ depending only on $d$. This proves the proposition.
\end{proof}
		
We are now able to prove Proposition \ref{prop:thm1}.	
\begin{proof}[Proof of Proposition \ref{prop:thm1}] 
By adding the intermediate term	$F(V_n(f))=\mu_{F,\phi_n}({E(f)})$, there holds	
\begin{align*}
&\sup_{f\in K}\left|F(f)- D\circ {E}(f)\right| \\
&\quad\leq \sup_{f\in K}|F(f)-F(V_n(f))| +\sup_{f\in K}|\mu_{F,\phi_n}({E}(f))- D\circ {E}(f)|\\
&\quad\leq\omega_F(\|f-V_n(f)\|_p)+ \sup_{y\in [-1,1]^{t_n}} |\mu_{F,\phi_n}(y)- D(y)| \\
&\quad\leq\omega_F(C_3 n^{-r})+{C_{n,d}}\omega_F(C_{n,d}n^{(d-1)(\frac{1}{2}-\frac{1}{p})}\CN^{-\frac{2}{t_n}}).
\end{align*}
The second inequality is derived from \lemref{lem:bestapproximationerror}, and the final one follows from \eqref{equ:app_ctns} and \propref{prop:w-w}.
\end{proof}
	
\subsection{Proof for the encoder with discrete input}\label{subsec:proof_dis}
In this subsection, we focus on constructing the linear operator $\widehat{V}_{n,M}$.
For $\{\xi_j\}_{j=1}^M \subset \sph$ obeying the cubuture rule of $2m$, the inputs of $\widehat{E}$ are consisted by discrete function values $\bz_f = (f(\xi_1), f(\xi_2), \ldots, f(\xi_M))' \in \RR^M$ instead of continuous function $f \in W_p^r(\sph).$

Recall that $\widehat{V}_{n,M}(\bz_f)$ is defined as
\begin{equation*}
\widehat{V}_{n,M}(\bz_f) 
= \sum_{k=0}^\infty \sum_{l=1}^{C(k,d)} 
{\eta}\lf(\frac{k}{n}\r) 
\lf( \sum_{j=1}^M \lambda_j f(\xi_j)Y_{k,l}(\xi_j) \r) Y_{k,l}
=\sum_{k=0}^\fz \sum_{l=1}^{C(k,d)}
{\eta}\lf(\frac{k}{n}\r){\braket{f, Y_{k, l}}}_{Q_m} Y_{k, l},
\end{equation*}
and the encoder for discrete input is defined from $\RR^{M}$ to $\RR^{t_n}$ as
\begin{equation*}
\widehat{E}(\bz_f)=\phi_n \circ\widehat{V}_{n,M}(\bz_f) = \left(\langle \widehat{V}_{n,M}(\bz_f), Y_1\rangle, \langle \widehat{V}_{n,M}(\bz_f), Y_2 \rangle, \ldots, \langle \widehat{V}_{n,M}(\bz_f), Y_{t_n} \rangle\right)'.
\end{equation*}
Each component of $\widehat{E}(\beta_f)$ can be bounded by $\|f\|_{W_p^r(\sph)} \leq 1$, as shown by \eqref{equ:c_bound}.

The following lemma from \cite{wang2017filtered} provides the approximation rate of $\widehat{V}_{n,M}(\bz_f)$.

\begin{lem}[\cite{wang2017filtered}]\label{lem:discrete_ln_app}
Let $d,m\in \NN_+$, $n={(2m+1)}/{3}$ and { $M\asymp m^{d-1}$}.
For $f\in W^r_p (\mathbb{S}^{d-1})$ with $1\leq p\leq \infty$ and $r>(d-1)/p$, we have
	\begin{equation}\label{equ:discrete_ln_app}
		\left\|f - \widehat{V}_{n,M}(\bz_f)\right\|_{L_p(\sph)} \leq C_4 n^{-r} \left\|f\right\|_{W^r_p(\sph)},
	\end{equation}
	where $C_4$ is a constant depending  on $d$. 
\end{lem}


The forthcoming demonstration is analogical to the proof of Proposition \ref{prop:thm1},
that is to estimate $\sup_{f\in K} |F(f)-F(\widehat{V}_{n,M}(\bz_f))|$ and $\sup_{f\in K}|\mu_{F,\phi_n}(\widehat{E}(\bz_f))- D\circ \widehat{E}(\bz_f))|$.

\begin{proof}[Proof of \thmref{thm:discreteinput}]
    Applying equation \eqref{equ:app_ctns}, Proposition \ref{prop:w-w} and Lemma \ref{lem:discrete_ln_app}, we have 
\begin{align*}
&\sup_{f\in K}|F(f)- D\circ \widehat{E}(\bz_f)|\\
&\quad\leq\sup_{f\in K} |F(f)-F(\widehat{V}_{n,M}(\bz_f))|+ \sup_{f\in K}|F(\widehat{V}_{n,M}(\bz_f))- D\circ \widehat{E}(\bz_f)| \\
&\quad\leq\sup_{f \in K}\omega_F(\|f-\widehat{V}_{n,M}(\bz_f)\|_p)+ \sup_{f\in K}|\mu_{F,\phi_n}(\widehat{E}(\bz_f))- D\circ \widehat{E}(\bz_f))| \\
&\quad\leq\omega_F\left(C_4n^{-r} \right)+\sup_{y \in [-1,1]^{t_n}}\left| \mu_{F,\phi_n}(y)-D(y))\right| \\
&\quad\leq\omega_F(C_4n^{-r})+{C_{n,d}}\omega_F\left(C_{n,d}n^{(d-1)(\frac{1}{2}-\frac{1}{p})}\CN^{-\frac{2}{t_n}}\right).
\end{align*}
This completes the proof of \thmref{thm:discreteinput}.
\end{proof}

\subsection{Proof for the encoder with discrete and noisy input}\label{proof_subsec:dis&noise}
In this subsection,
we present the proof of lemmas and theorem in the case when the encoder's input is the discrete function values with some noise. 

{Recall that with noise $\epsilon=\{\epsilon_j\}_{j=1}^{M'}$ being i.i.d  random variables satisfying $\EE[\epsilon_j]=0$ and ranging in $[-1,1]$,
we define the linear operator $\widetilde{V}_{n,M'}$ for $\bz_{f,\epsilon}=\{\epsilon_j+f(\xi_j)\}_{j=1}^{M'}$ as
\begin{equation*}
\widetilde{V}_{n,M'}(\bz_{f,\epsilon}) 
= \sum_{k=0}^\infty \sum_{l=1}^{C(k,d)} 
{\eta}\lf(\frac{k}{n}\r)\left( \sum_{j=1}^{M'}\lambda_j\left(f(\xi_j)+\epsilon_j\right)Y_{k,l}(\xi_j)\right)Y_{k,l},
\end{equation*}
which is also a random variable due to the randomness of $\ez$.
The encoder for discrete input with noise is given by
\begin{equation*}
\widetilde{E}(\beta_{f,\epsilon})= \phi_n \circ \widetilde{V}_{n,M'}(\beta_{f,\epsilon})
= \left(\langle \widetilde{V}_{n,M'}(\bz_{f,\ez}), Y_1\rangle, \langle \widetilde{V}_{n,M'}(\bz_{f,\ez}), Y_2 \rangle, \ldots, \langle\widetilde{V}_{n,M'}(\bz_{f,\ez}), Y_{t_n} \rangle\right)'.
\end{equation*}
}

The following \lemref{lem:ineq_noise} plays an essential role in proving \lemref{lem:dis&noise}, which is an extension of  \cite[Lemma 6.2]{gia2009localized}. 
Before presenting the \lemref{lem:ineq_noise},
we first show some important properties of the zonal function and use it to reformulate the definition in \eqref{def:dis-n}.
For a zonal function $\phi_{q}$ with $q\geq1$ defined as
$$\phi_{q}(x,\xi)
:=\sum^{\fz}_{k=0}\etat \lf(\frac{k}{q}\r)\sum^{C(k,d)}_{l=1}
Y_{k, l}(x)Y_{k, l}(\xi), \quad \ x,\xi\in\sph,$$
it holds that
\begin{equation}\label{equ:range_int_phi2}
    \int_{\sph} |\phi_{q}(x,\xi)| ^2 \,\md\sz_d(\xi) \asymp q^{d-1} \asymp \sup_{\xi \in\sph} |\phi_{q}(x,\xi)| = \phi_q(1).
\end{equation}
In fact, by the properties of the filter $\etat$ and orthonormal basis $Y_{k,l}$, we first have
\begin{align*}
\int_{\sph} |\phi_{q}(x,\xi)| ^2\,\md\sz_d(\xi)
&=\sum_{k=0}^\fz \etat\lf(\frac{k}{q}\r)^2
\int_{\sph}\sum^{C(k,d)}_{l=1}
Y_{k, l}^2(x)Y_{k, l}^2(\xi)\,\md\sz_d(\xi)\\
&=\sum^{2q}_{k=0} \etat\lf(\frac{k}{q}\r)^2 \sum^{C(k,d)}_{l=1}
Y_{k, l}^2(x) \asymp q^{d-1}.
\end{align*}
Then, by using the Cauchy–Schwarz inequality, we obtain
\begin{align*}
\phi_q(1)=\phi_q(x\cdot x) 
&\leq\sup_{\xi\in\sph}|\phi_{q}(x,\xi)|\\
 &\leq \sup_{\xi\in\sph} \sum^{2q}_{k=0} \etat\lf(\frac{k}{q}\r)
\lf(\sum^{C(k,d)}_{l=1} Y_{k, l}^2(x) \r)^{\frac{1}{2}}
\lf(\sum^{C(k,d)}_{l=1} Y_{k, l}^2(\xi) \r)^{\frac{1}{2}}\\
&=\sum^{2q}_{k=0} \etat\lf(\frac{k}{q}\r) C(k,d)=\phi_q(1)\asymp q^{d-1},
\end{align*}
where we use the fact $\sum^{C(k,d)}_{l=1} Y_{k, l}^2(x) = C(k,d)$ to deduce the last line.

Furthermore, 
by taking $q=n$ in zonal function,
the definition of the discrete operator with random noise in \eqref{def:dis-n} could be reformulated as
\begin{align}\label{def:dis-n2}
\widetilde{V}_{n,M'}(\beta_{f,\epsilon})(x) 
=\sum_{j=1}^{M'}\lambda_j (f(\xi_j)+\epsilon_j) \phi_{n}(\xi_j,x).
\end{align}

Moreover, for $x \in \{\xi_1, \ldots, \xi_{M'}\}$ and $\{(\lambda_j,\xi_j)\}_{j=1}^{M'}$ obeying the cubature rule of degree $2m$, if we take $p=2$ in \eqref{equ:m-z}, there holds
\begin{align*}
 |\lz_j| \cdot \phi_{m}^2(1) 
\leq \sum_{j=1}^{M'} |\lz_j| \cdot \phi_{m}^2(x,\xi_j)
\asymp\int_{\sph} \phi_{m}^2(x,\xi)\,\md{\sz_d}(\xi).
\end{align*} 
This estimate, together with \eqref{equ:range_int_phi2}, implies that the weights $\{\lambda_j\}$ have the following bound:
\begin{equation}\label{equ:range_lbd}
    |\lz_j|\asymp  m^{-(d-1)}, \quad j=1, \ldots, M'.
\end{equation}

We are now ready to demonstrate \lemref{lem:ineq_noise}, which is crucial for proving \lemref{lem:dis&noise}. To achieve this, set $ Z_{\epsilon_j}(x) := m^{d-1} \lambda_j \epsilon_j \phi_n(\xi_j, x)$ for $\{(\lambda_j,\xi_j)\}_{j=1}^{M'}$ satisfying the cubature rule of degree $2m$. 
Inspired by the work of \cite[Lemma 6.2]{gia2009localized}, we provide a comprehensive proof below for the sake of completeness.

\begin{lem}\label{lem:ineq_noise}
Let $d, n, M'\in\nn_+$, $B>0$, and let $\{\ez_j\}_{j=1}^{M'}$ be independent random variables ranging in $[-1,1]$ with zero means. 
Suppose  $Z_{\ez_j}(x)\in \Pi_{2n}^{d-1}(\sph)$ satisfies $\max_{1\leq j \leq M', \ x\in\sph} |Z_{\epsilon_j}(x)|\leq  n^{d-1}$ and {$\sum_{j=1}^{M'}\var\lf( Z_{\epsilon_j}\r)$} is uniformly bounded by $B n^{d-1}$ on $\sph$ for $j=1, \ldots, M'$.
If $\delta>0$ satisfies $12  (\delta+{d-1}) n^{d-1} \log n\leq B$,
then for $1\leq p\leq\infty$, we have
\begin{equation}\label{6.2}
	\Prob \lf( \lf\| \sum\limits^{M'}_{j=1} Z_{\epsilon_j}\,\r\|_p \geq 2\sqrt{3 B (\delta+{d-1}) n^{d-1} \log n}\r) \leq C_6n^{-\delta}.
\end{equation}
Here, the positive constant $C_6$ is independent of $M'$ and the distributions of $\epsilon_j$.
\end{lem}
{
\begin{proof}
Bennett’s inequality \cite[p.192]{pollard2012convergence} stated that if \(\{ Y_1, \dots, Y_{M'} \}\) are independent random variables with zero means,  \( Y_i \in [-L,L]\) for all $i\in\{1,\dots,M'\}$, and \( \sum_{i=1}^{M'} \text{Var}(Y_i)\leq V\), then for \( \varepsilon > 0 \),
\[
\Prob\left\{ \lf|\sum_{i=1}^{M'} Y_i\r| \geq \varepsilon \right\} \leq 2 \exp\left(-\frac{V}{L^2}g\lf( \frac{L\varepsilon}{V}\r)\right),
\]
where $g(x) = (1 + x) \log(1 + x) - x \text{ for } x> 0.$
Notice that
$g(x) = \int_0^x \int_0^u (1 + w)^{-1} \, dw \, du $,
which implies
\( g(x) \geq x^2/3 \) if \( 0 \leq x \leq 1/2 \).
Thus, if \( L\varepsilon/V \leq 1/2 \), we have
\begin{equation*}
\text{Prob} \left\{ \left| \sum_{j=1}^{M'} Y_j \right| \geq \varepsilon \right\} \leq 2 \exp\left(-\frac{\varepsilon^2}{3V} \right). 
\end{equation*}
By taking \(Y_j= Z_{\epsilon_j}(x) \), \( L =n^{d-1} \), \( V =Bn^{d-1} \), and \( \varepsilon = \sqrt{3B(\delta+ d-1)n^{d-1} \log n} \), it easy to verify that \( L\varepsilon/V \leq 1/2 \). Therefore,
\begin{equation}\label{equ:prob_each}
		\Prob \lf\{\lf| \sum\limits^{M'}_{j=1}Z_{\epsilon_j}(x)\r| \geq \sqrt{3 B (\delta+{d-1}) n^{d-1} \log n}\r\} \leq 2 n^{-\delta-(d-1)}.
	\end{equation}

Next, we give a uniform probability estimation. 
{
For a covering set of $\sph$ with centers as a set $\CC \subset \sph$ and radius of $1/(4n)$,  the cardinality of $\CC$ will be $|\CC|\asymp  n^{d-1}$.
Suppose 
$$x^* \in \arg\sup_{x \in \sph}\lf|\sum_{j=1}^{M'} Z_{\epsilon_j}(x)\r|,$$
then there exists a center $x_i \in \CC$, such that $x^* \in \CB(x_i,1/(4n))$, where $\CB(x_i,1/(4n))$ means a ball centered at $x_i$ with a radius of $1/(4n)$.
Since $\sum_{j=1}^{M'} Z_{\epsilon_j}(x)\in\Pi_{2n}^{d-1}(\sph)$,
then by the Bernstein inequality, we obtain
\[\lf| \sum_{j=1}^{M'} Z_{\epsilon_j}(x_i)-\sum_{j=1}^{M'} Z_{\epsilon_j}(x^*)\r| \leq 2n \lf\|\sum_{j=1}^{M'} Z_{\epsilon_j}\r\|_\infty \|x_i-x^*\|\leq \frac{1}{2} \lf\|\sum_{j=1}^{M'} Z_{\epsilon_j}\r\|_\infty. \]
Thus, we have 
$
\lf|\sum_{j=1}^{M'} Z_{\epsilon_j}(x_i) \r| \geq 
\frac{1}{2}\lf\|\sum_{j=1}^{M'} Z_{\epsilon_j}\r\|_\fz.
$
Moreover, with $|\CC|\leq C_2 n^{d-1}$ and all the possibilities  $x^*$ may belonging to, we have}
\begin{align*}
&\Prob\lf\{  \lf\|\sum_{j=1}^{M'} Z_{\epsilon_j}\r\|_\infty \geq 2\sqrt{3 B (\delta+{d-1}) n^{d-1} \log n}  \r\}\\
&\quad\leq\Prob\lf\{ \lf|\sum_{j=1}^{M'} Z_{\epsilon_j}(x_i) \r| \geq \sqrt{3 B (\delta+{d-1}) n^{d-1} \log n}  \r\}\\
&\quad\leq|\CC| \cdot  \Prob\lf\{ \lf|\sum_{j=1}^{M'} Z_{\epsilon_j}(x) \r| \geq \sqrt{3 B (\delta+{d-1}) n^{d-1} \log n}  \r\}
\leq2C_2n^{-\delta}.
\end{align*}
Finally, following H{\"o}lder inequality, for $1\leq p\leq\infty$, we have $\|f\|_p \leq \| f\|_\infty$ for all $f \in \Pi_{2n}^{d-1}(\sph)$.
By taking $C_6=2 C_2$, we could complete this proof.
\end{proof}
}
\begin{remark}
In Lemma \eqref{lem:ineq_noise}, we restrict the range of $\{\ez_j\}^{M'}_{j=1}$ within $[-1,1]$ for simplicity. But the condition $[-1,1]$ could be reduced to $[-R,R]$ for any $R>0$.
\end{remark}

Now we are ready to prove Lemma \ref{lem:dis&noise}.
\begin{proof}[Proof of Lemma \ref{lem:dis&noise}]
For $\{(\lambda_j,\xi_j)\}_{j=1}^{M'}$ following the cubature rule of degree $2m$, denote $\widetilde{V}_{n,M'}(\epsilon):= \sum_{j=1}^{M'}\lambda_j\epsilon_j\phi_n(\xi_j,x)$.
Since
$$\lf\| \widetilde{V}_{n,M'}( \bz_{f,\epsilon})-f \r\|_p 
\leq \|\widehat{V}_{n,M'}(\bz_f)-f \|_p + \|\widetilde{V}_{n,M'}(\epsilon)\|_p,$$ and 
$\|\widehat{V}_{n,M'}(\bz_f)-f \|_p \leq {C_4}n^{-r}$ 
given by \lemref{lem:discrete_ln_app}, we could assert the occurrence of the event $\lf\|\,\widetilde{V}_{n,M'}( \bz_{f,\epsilon})-f\,\r\|_p \geq C_5 n^{-r}$ could guarantee the occurrence of the event $\|\widetilde{V}_{n,M'}(\epsilon)\|_p \geq (C_5-C_4)n^{-r}:=c_3n^{-r}$, which implies
\begin{equation*}
\Prob \lf( \lf\| \widetilde{V}_{n,M'}( \bz_{f,\epsilon})-f \r\|_p \geq C_5 n^{-r} \r)
\leq
\Prob \left(\|\widetilde{V}_{n,M'}(\epsilon)\|_p \geq c_3n^{-r}  \right).
\end{equation*}

Next, we aim to show $\Prob \lf(\|\widetilde{V}_{n,M'}(\epsilon)\|_p \geq c_3 n^{-r} \r)$ could be further bounded by $C_6 n^{-\delta}$.
If we choose $$Z_{\epsilon_j}(x)= {m}^{d-1}\lz_j \ez_j\phi_n(\xi_j,x),$$ 
it is easy to check that $\EE[Z_{\epsilon_j}(x)]=0$ as $\EE[\epsilon_j]=0$ {and 
$\|Z_{\epsilon_j}\|_\fz \leq  n^{d-1}$} by \eqref{equ:range_int_phi2} and \eqref{equ:range_lbd}.
Since $\{\ez_j\}^{M'}_{j=1}$ are sampled within $[-1,1]$,  then the variance $\var\lf(\epsilon_j\r)\leq 1$ for all $j=1,\ldots,M'$. 
Thus, together with the Marcinkiewicz-Zygmund inequality, we obtain
\begin{equation*}
\begin{split}
\sum_{j=1}^{M'} m^{2(d-1)} \lz_j^2 \phi_{n}^2(x,\xi_j)
	\leq & c_4 \sum_{j=1}^{M'} m^{d-1}  |\lz_j|  \phi_{n}^2(x,\xi_j)\\
	\leq & c_4 m^{d-1} \int_{\sph}  \phi_{n}^2(x,\xi_j) \,\md\sz_d(\xi_j)
	\leq c_4 (mn)^{d-1},
\end{split}
\end{equation*}
where $c_4$ comes from the estimate of $\lambda_j$.
{
Therefore, 
\begin{equation*}
\begin{split}
    \sum_{j=1}^{M'} \var\lf(Z_{\ez_j}\r)=\sum_{j=1}^{M'} {m}^{2(d-1)}\lambda_j^2 \phi_n^2(x,\xi_j)\var\lf(\ez_j\r) \leq c_4 (mn)^{d-1}.
\end{split}
\end{equation*}
By the definition of $\{Z_{\ez_j}\}$, we notice that  $\widetilde{V}_{n,M'}(\ez) = \frac{1}{m^{d-1}}\sum_{j=1}^{M'}Z_{\epsilon_j}.$
So by taking $B = c_4 {m}^{d-1}$ in \lemref{lem:ineq_noise}, we obtain
\begin{equation*}
\Prob \lf(\,\|\widetilde{V}_{n,M'}(\ez)\|_p \geq 2\sqrt{\frac{3  (\dz+d-1) n^{d-1} \log n}{{ {m}^{d-1}}}}\,\r) \leq C_6 n^{-\dz}. 
\end{equation*}
}
Finally, we may choose $m$ appropriately such that $12(\dz+d-1)\log n\cdot n^{2r+d-1} \leq {m}^{d-1}$ to conclude that
$$
\Prob \left( \left\| \widetilde{V}_{n,M'}( \bz_{f,\epsilon})-f \right\|_p \geq C_5 n^{-r} \right) \leq \Prob \left(\|\widetilde{V}_{n,M'}(\epsilon)\|_p \geq c_3 n^{-r}  \right) \leq C_6 n^{-\delta}.
$$
This proves \lemref{lem:dis&noise}.
\end{proof}


We are now ready to prove \thmref{thm:discretenoise}.

\begin{proof}[proof of \thmref{thm:discretenoise}]
Following the analogical ideas of previous proofs, we plug in the intermediate term $F(\widetilde{V}_{n,M'}(\bz_{f,\epsilon}))$ and obtain
\begin{align*}
&\sup_{f\in K}|F(f)- D \circ \widetilde{E}(\bz_{f,\epsilon})|\\
&\quad\leq\sup_{f \in K}|F(f)-F(\widetilde{V}_{n,M'}(\bz_{f,\epsilon}))| + \sup_{f \in K}|F(\widetilde{V}_{n,M'}(\bz_{f,\epsilon}))- D \circ \widetilde{E}(\bz_{f,\epsilon})| \\
&\quad\leq\omega_{F}(\|f-\widetilde{V}_{n,M'}(\bz_{f,\epsilon})\|)+ \sup_{y\in[-R',R']^{t_n}}\left|\mu_{F,\phi_n}(y)- D\left(y \right)\right|.
\end{align*}
The first item of the inequality could be further bounded by $\omega_F({C_5}n^{-r})$ with confidence $1-{C_6} n^{-\dz}$ by \lemref{lem:dis&noise}. 
For the second item, it follows from \eqref{equ:app_ctns} and Proposition \ref{prop:w-w}, 
\begin{align*}
\sup_{y\in[-R',R']^{t_n}}|\mu_{F,\phi_n}(y)- D(y)|
\leq C_{n,d}\omega_{\mu_{F,\phi_n}}\left(C_{n,d,R'} \CN^{-\frac{2}{t_n}} \right)
\leq C_{n,d}\omega_F\left(C_{n,d,R'}n^{(d-1)(\frac{1}{2}-\frac{1}{p})}\CN^{-\frac{2}{t_n}}\right),
\end{align*}
which completes the proof of \thmref{thm:discretenoise}.
\end{proof}

	
\section{Conclusion}\label{sec:conclusion}
This paper provides a comprehensive toolkit for approximating continuous functionals defined on $W_p^r(\sph)$ with $1 \leq p \leq \infty$. Three different encoders have been progressively defined to be more in line with practical applications, and the corresponding approximation rates have been established for continuous functionals. We first illustrate the idea of constructing the encoder-decoder network framework to accommodate the infinite-dimensional nature of functional's domain. The encoder $E$ is the composition of a linear operator $V_n: L_p \to \Pi_{2n}^{d-1}(\sph)$ and an isometric isomorphism map $\phi_n :\Pi_{2n}^{d-1}(\sph) \to \RR^{t_n}$, which provides the fully connected neural network decoder $D$ a vectorized input. A further encoder $\widehat{E}=\phi_n \circ \widehat{V}_{n,M}: \RR^M \to \RR^{t_n}$ is given to understand the case when the real-world objects are discretized, where the operator $\widehat{V}_{n,M}: \RR^M \to \Pi_{2n}^{d-1}(\sph)$ is inspired by the cubature formula. Last but not least, we analyze the approximation accuracy in the presence of noise on input objects. Therefore, the input of encoder $\widetilde{E}=\phi_n \circ \widetilde{V}_{n,M'}$ is a set of random variables in $\RR^{M'}$. The decoders of the above situations are all fully connected neural networks mapping from $\RR^{t_n}$ to $\RR$, and we give the approximation rates using our encoder-decoder networks for continuous functional defined on $W_p^r(\sph)$.

{\renewcommand\arraystretch{1.3}
\begin{table}[ht] 
\centering
\begin{tabular}{>{\centering\arraybackslash}p{2cm}|>{}p{12cm}}
\hhline{==}
\textbf{Notation} & \textbf{Description}  \\ 
\hline
$d$   &  the dimension of the unit sphere $\mathbb{S}^{d-1}$.            \\ 
\hline
$n$     &   the degree of the linear operators   $V_{n}$, $\widehat{V}_{n,M}$ and $\widetilde{V}_{n,M'}$.            \\ 
\hline
$m$   & the degree of cubature formulae to discrete the integration.\\
\hline
$J$      &     the depth of DNN.      \\
\hline
$\CN$   & number of free parameters in DNN.  \\
\hline
$t_n$     &        the input dimension of DNN.       \\
\hline
$E$  & an encoder for continuous functionals.\\
\hline
$\widehat{E}$ & an encoder for discrete input.\\
\hline
$\widetilde{E}$ & an encoder for discrete input with noise.\\
\hline
$M$       &  the number of spherical points for encoder $\widehat{E}$.     \\ 
\hline
$M'$       & the number of spherical points for the encoder $\widetilde{E}$.    \\  
\hline
$V_n$       &   the linear operator of degree $n$.     \\
\hline
${ \widehat{V}_{n,M}}$   &  the linear operator with discrete  input.       \\
\hline
$\widetilde{V}_{n,M'}$       &      the   linear operator with discrete and noisy input.   \\
\hhline{==}
\end{tabular}
\caption{Notations in the paper}
\label{table:notation}
\end{table}
}



\newpage
\bibliography{ref}
\bibliographystyle{apalike}
	
\end{document}

%% file: macros.tex
\def\proj{\operatorname{Proj}}  
\DeclareMathOperator{\var}{Var}

\def\Prob{\operatorname{Prob}}

\def\CB{\mathcal B} 
\def\CC{\mathcal C} 
\def\CM{\mathcal M} 
\def\CN{\mathcal N} 

\def\EE{\mathbb{E}} 
\def\RR{\mathbb{R}} 
\def\NN{\mathbb{N}} 
\def\ZZ{\mathbb{Z}} 
\def\bz{\mathbf{z}} 

\def\f{\frac}      
\def \md {\mathrm{d}} 

\def\sph{\mathbb{S}^{d-1}} 

\def\s{\sigma}

\def\l{\lambda}

\def\rr{\mathbb{R}}

\def\nn{\mathbb{N}}

\def\lz{\lambda}

\def\dz{\delta}

\def\ez{\epsilon}

\def\bz{\beta}

\def\sz{\sigma}

\def\r{\right}
\def\lf{\left}
\def\fz{\infty}

\def\etat{{\eta}}

\newtheorem{thm}{Theorem}
\newtheorem{lem}{Lemma}
\newtheorem{prop}{Proposition}

\newtheorem{example}{Example}
\newtheorem{remark}{Remark}
\theoremstyle{definition}
\newtheorem{definition}{Definition}

\newcommand{\thmref}[1]{Theorem~\ref{#1}}

\newcommand{\lemref}[1]{Lemma~\ref{#1}}

\newcommand{\propref}[1]{Proposition~\ref{#1}}

%% file: ref.bbl
\begin{thebibliography}{}

\bibitem[Bhattacharya et~al., 2021]{bhattacharya2021model}
Bhattacharya, K., Hosseini, B., Kovachki, N.~B., and Stuart, A.~M. (2021).
\newblock Model reduction and neural networks for parametric pdes.
\newblock {\em The SMAI Journal of Computational Mathematics}, 7:121--157.

\bibitem[Brown and Dai, 2005]{brown2005approximation}
Brown, G. and Dai, F. (2005).
\newblock Approximation of smooth functions on compact two-point homogeneous spaces.
\newblock {\em Journal of Functional Analysis}, 220(2):401--423.

\bibitem[Chen et~al., 2011]{chen2011single}
Chen, D., Hall, P., and M{\"u}ller, H.-G. (2011).
\newblock Single and multiple index functional regression models with nonparametric link.
\newblock {\em The Annals of Statistics}, 39(3):1720--1747.

\bibitem[Chen and Chen, 1995]{chen1995universal}
Chen, T. and Chen, H. (1995).
\newblock Universal approximation to nonlinear operators by neural networks with arbitrary activation functions and its application to dynamical systems.
\newblock {\em IEEE Transactions on Neural Networks}, 6(4):911--917.

\bibitem[Dai and Xu, 2013]{dai2013approximation}
Dai, F. and Xu, Y. (2013).
\newblock {\em Approximation theory and harmonic analysis on spheres and balls}, volume~23.
\newblock Springer.

\bibitem[D{\~u}ng, 2011]{dung2011optimal}
D{\~u}ng, D. (2011).
\newblock Optimal adaptive sampling recovery.
\newblock {\em Advances in Computational Mathematics}, 34(1):1--41.

\bibitem[Fang et~al., 2020]{fang2020theory}
Fang, Z., Feng, H., Huang, S., and Zhou, D.-X. (2020).
\newblock Theory of deep convolutional neural networks ii: Spherical analysis.
\newblock {\em Neural Networks}, 131:154--162.

\bibitem[Feng et~al., 2023]{feng2023generalization}
Feng, H., Huang, S., and Zhou, D.-X. (2023).
\newblock Generalization analysis of cnns for classification on spheres.
\newblock {\em IEEE Transactions on Neural Networks and Learning Systems}, 34(9):6200--6213.

\bibitem[Gia and Mhaskar, 2009]{gia2009localized}
Gia, Q.~L. and Mhaskar, H. (2009).
\newblock Localized linear polynomial operators and quadrature formulas on the sphere.
\newblock {\em SIAM Journal on Numerical Analysis}, 47(1):440--466.

\bibitem[He et~al., 2016]{he2016deep}
He, K., Zhang, X., Ren, S., and Sun, J. (2016).
\newblock Deep residual learning for image recognition.
\newblock In {\em Proceedings of the IEEE Conference on Computer Vision and Pattern Recognition}, pages 770--778.

\bibitem[Hinton et~al., 2012]{hinton2012deep}
Hinton, G., Deng, L., Yu, D., Dahl, G.~E., Mohamed, A.-r., Jaitly, N., Senior, A., Vanhoucke, V., Nguyen, P., Sainath, T.~N., et~al. (2012).
\newblock Deep neural networks for acoustic modeling in speech recognition: The shared views of four research groups.
\newblock {\em IEEE Signal Processing Magazine}, 29(6):82--97.

\bibitem[Kovachki et~al., 2023]{kovachki2023neural}
Kovachki, N., Li, Z., Liu, B., Azizzadenesheli, K., Bhattacharya, K., Stuart, A., and Anandkumar, A. (2023).
\newblock Neural operator: Learning maps between function spaces with applications to pdes.
\newblock {\em Journal of Machine Learning Research}, 24(89):1--97.

\bibitem[Kovachki et~al., 2024]{kovachki2024operator}
Kovachki, N.~B., Lanthaler, S., and Stuart, A.~M. (2024).
\newblock Operator learning: Algorithms and analysis.
\newblock {\em arXiv preprint arXiv:2402.15715}.

\bibitem[Liu et~al., 2024]{liu2024generalization}
Liu, H., Dahal, B., Lai, R., and Liao, W. (2024).
\newblock Generalization error guaranteed auto-encoder-based nonlinear model reduction for operator learning.
\newblock {\em arXiv:2401.10490}.

\bibitem[Lu et~al., 2021]{lu2021learning}
Lu, L., Jin, P., Pang, G., Zhang, Z., and Karniadakis, G.~E. (2021).
\newblock Learning nonlinear operators via deeponet based on the universal approximation theorem of operators.
\newblock {\em Nature Machine Intelligence}, 3(3):218--229.

\bibitem[Maggioni and Mhaskar, 2008]{maggioni2008diffusion}
Maggioni, M. and Mhaskar, H.~N. (2008).
\newblock Diffusion polynomial frames on metric measure spaces.
\newblock {\em Applied and Computational Harmonic Analysis}, 24(3):329--353.

\bibitem[Mao et~al., 2021]{mao2021theory}
Mao, T., Shi, Z., and Zhou, D.-X. (2021).
\newblock Theory of deep convolutional neural networks iii: Approximating radial functions.
\newblock {\em Neural Networks}, 144:778--790.

\bibitem[Montavon et~al., 2011]{montavon2011kernel}
Montavon, G., Braun, M.~L., and M{\"u}ller, K.-R. (2011).
\newblock Kernel analysis of deep networks.
\newblock {\em Journal of Machine Learning Research}, 12(9).

\bibitem[Mont{\'u}far and Wang, 2022]{montufar2022distributed}
Mont{\'u}far, G. and Wang, Y.~G. (2022).
\newblock Distributed learning via filtered hyperinterpolation on manifolds.
\newblock {\em Foundations of Computational Mathematics}, 22(4):1219--1271.

\bibitem[Morris, 2015]{morris2015functional}
Morris, J.~S. (2015).
\newblock Functional regression.
\newblock {\em Annual Review of Statistics and Its Application}, 2:321--359.

\bibitem[Poggio et~al., 2017]{poggio2017and}
Poggio, T., Mhaskar, H., Rosasco, L., Miranda, B., and Liao, Q. (2017).
\newblock Why and when can deep-but not shallow-networks avoid the curse of dimensionality: a review.
\newblock {\em International Journal of Automation and Computing}, 14(5):503--519.

\bibitem[Pollard, 2012]{pollard2012convergence}
Pollard, D. (2012).
\newblock {\em Convergence of stochastic processes}.
\newblock Springer Science \& Business Media.

\bibitem[Schmidt-Hieber, 2020]{schmidt2020nonparametric}
Schmidt-Hieber, J. (2020).
\newblock Nonparametric regression using deep neural networks with relu activation function.
\newblock {\em The Annals of Statistics}, 48(4):1875--1897.

\bibitem[Seidman et~al., 2022]{seidman2022nomad}
Seidman, J., Kissas, G., Perdikaris, P., and Pappas, G.~J. (2022).
\newblock Nomad: Nonlinear manifold decoders for operator learning.
\newblock {\em Advances in Neural Information Processing Systems}, 35:5601--5613.

\bibitem[Siegel, 2023]{siegel2023optimal}
Siegel, J.~W. (2023).
\newblock Optimal approximation rates for deep relu neural networks on sobolev and besov spaces.
\newblock {\em Journal of Machine Learning Research}, 24(357):1--52.

\bibitem[Song et~al., 2023a]{song2023approximation(a)}
Song, L., Fan, J., Chen, D.-R., and Zhou, D.-X. (2023a).
\newblock Approximation of nonlinear functionals using deep relu networks.
\newblock {\em Journal of Fourier Analysis and Applications}, 29(4):50.

\bibitem[Song et~al., 2023b]{song2023approximation(b)}
Song, L., Liu, Y., Fan, J., and Zhou, D.-X. (2023b).
\newblock Approximation of smooth functionals using deep relu networks.
\newblock {\em Neural Networks}, 166:424--436.

\bibitem[Suzuki, 2019]{suzuki2019adaptivity}
Suzuki, T. (2019).
\newblock Adaptivity of deep relu network for learning in besov and mixed smooth besov spaces: optimal rate and curse of dimensionality.
\newblock In {\em International Conference on Learning Representations}.

\bibitem[Vaswani et~al., 2017]{vaswani2017attention}
Vaswani, A., Shazeer, N., Parmar, N., Uszkoreit, J., Jones, L., Gomez, A.~N., Kaiser, {\L}., and Polosukhin, I. (2017).
\newblock Attention is all you need.
\newblock {\em Advances in Neural Information Processing Systems}, 30.

\bibitem[Wang and Sloan, 2017]{wang2017filtered}
Wang, H. and Sloan, I.~H. (2017).
\newblock On filtered polynomial approximation on the sphere.
\newblock {\em Journal of Fourier Analysis and Applications}, 23:863--876.

\bibitem[Yang and Zhou, 2024]{yang2024nonparametric}
Yang, Y. and Zhou, D.-X. (2024).
\newblock Nonparametric regression using over-parameterized shallow relu neural networks.
\newblock {\em Journal of Machine Learning Research}, 25:1--35.

\bibitem[Yarotsky, 2017]{yarotsky2017error}
Yarotsky, D. (2017).
\newblock Error bounds for approximations with deep relu networks.
\newblock {\em Neural Networks}, 94:103--114.

\bibitem[Yarotsky, 2018]{yarotsky2018optimal}
Yarotsky, D. (2018).
\newblock Optimal approximation of continuous functions by very deep relu networks.
\newblock In {\em Conference on Learning Theory}, pages 639--649. PMLR.

\bibitem[Zhou, 2020]{zhou2020universality}
Zhou, D.-X. (2020).
\newblock Universality of deep convolutional neural networks.
\newblock {\em Applied and Computational Harmonic Analysis}, 48(2):787--794.

\bibitem[Zhou et~al., 2024]{zhou2024approximation}
Zhou, T.-Y., Suh, N., Cheng, G., and Huo, X. (2024).
\newblock Approximation of rkhs functionals by neural networks.
\newblock {\em arXiv:2403.12187}.

\end{thebibliography}
